\def\year{2021}\relax
\documentclass[letterpaper]{article} % DO NOT CHANGE THIS
\usepackage{aaai21}  % DO NOT CHANGE THIS
\usepackage{times}  % DO NOT CHANGE THIS
\usepackage{helvet} % DO NOT CHANGE THIS
\usepackage{courier}  % DO NOT CHANGE THIS
\usepackage[hyphens]{url}  % DO NOT CHANGE THIS
\usepackage{graphicx} % DO NOT CHANGE THIS
\usepackage{natbib}  % DO NOT CHANGE THIS AND DO NOT ADD ANY OPTIONS TO IT
\usepackage{caption} % DO NOT CHANGE THIS AND DO NOT ADD ANY OPTIONS TO IT
\usepackage[utf8]{inputenc} % allow utf-8 input
\usepackage{booktabs}       % professional-quality tables
\usepackage{amsfonts}       % blackboard math symbols
\usepackage{nicefrac}       % compact symbols for 1/2, etc.
\usepackage{microtype}      % microtypography
\usepackage{amssymb,amsmath,bm,amsthm}
\usepackage{subcaption}
\usepackage[dvipsnames]{xcolor}
\usepackage{algorithm}
\usepackage{algpseudocode}
\usepackage{enumerate}
\usepackage{comment}

\def\th{\bm \theta}

\def\e{\bm{e}}
\def\ee{\bm{y}}

\def\z{\bm{z}}

\def\a{\alpha}

\def\S{\mathcal{S}}
\def\A{\mathcal{A}}
\def\L{\mathcal{L}}
\def\reals{\mathbb{R}}

\def\u{\bm u}

\def\u{\bm u}
\def\d{\delta}

\def\w{\mathbf{w}}

\def\x{\mathbf{x}}
\def\y{\mathbf{y}}

\def\E#1{\mathbb{E}\left[\, #1 \,\right]}

\def\tr{^{\top}}

\def\X{\mathbf{X}}

\def\q{\bm q}
\def\l{\lambda}

\def\g{\gamma}

\def\({\left(}
\def\){\right)}

\definecolor{blue_colour}{rgb}{0., 0.3, 0.8}
\definecolor{orange_colour}{rgb}{0.85, 0.4, 0.}
\def\blue#1{\textcolor{blue_colour}{\textbf{#1}}}
\def\orange#1{\textcolor{orange_colour}{\textbf{#1}}}

\newtheorem{proposition}{Proposition}

\newtheorem{lemma}{Lemma}
\newtheorem{definition}{Definition}

\newtheorem{property}{Property}
\usepackage{thmtools} 
\usepackage{thm-restate}

\title{Expected Eligibility Traces}

\def\year{2021}\relax

\title{Expected Eligibility Traces}
\author {
    Hado van Hasselt\textsuperscript{\rm 1}
    ,
    Sephora Madjiheurem\textsuperscript{\rm 2}
    ,
    Matteo Hessel\textsuperscript{\rm 1}
    \\
    David Silver\textsuperscript{\rm 1}
    ,
    Andr{\'e} Barreto\textsuperscript{\rm 1}
    ,
    Diana Borsa\textsuperscript{\rm 1}\\
}

\affiliations{
    \textsuperscript{\rm 1} DeepMind\\
    \textsuperscript{\rm 2} University College London, UK
}

\begin{document}

\maketitle

\begin{abstract}
The question of how to determine which states and actions are responsible for a certain outcome is known as the \emph{credit assignment problem} and remains a central research question in reinforcement learning and artificial intelligence. 
\emph{Eligibility traces} enable efficient credit assignment to the recent sequence of states and actions experienced by the agent, but not to counterfactual sequences that could also have led to the current state.
In this work, we introduce \emph{expected eligibility traces}. Expected traces allow, with a single update, to update states and actions that could have preceded the current state, even if they did not do so on this occasion.
We discuss when expected traces provide benefits over classic (instantaneous) traces in temporal-difference learning, and show that sometimes substantial improvements can be attained. We provide a way to smoothly interpolate between instantaneous and expected traces by a mechanism similar to bootstrapping, which ensures that the resulting algorithm is a strict generalisation of TD($\lambda$). Finally, we discuss possible extensions and connections to related ideas, such as successor features.
\end{abstract}

\label{intro}

Appropriate credit assignment has long been a major research topic in artificial intelligence \citep{Minsky:63}.  To make effective decisions and understand the world, we need to accurately associate events, like rewards or penalties, to relevant earlier decisions or situations. This is important both for learning accurate predictions, and for making good decisions.

\emph{Temporal credit assignment} can be achieved with repeated temporal-difference (TD) updates \citep{Sutton:1988}.  One-step TD updates propagate information slowly: when a surprising value is observed, the state immediately preceding it is updated, but no earlier states or decisions are updated.
\emph{Multi-step} updates \citep{Sutton:1988,SuttonBarto:2018} propagate information faster over longer temporal spans, speeding up credit assignment and learning. Multi-step updates can be implemented online using \emph{eligibility traces} \citep{Sutton:1988}, without incurring significant additional computational expense, even if the time spans are long; these algorithms have computation that is independent of the temporal span of the prediction \citep{vanHasselt:2015}.

Traces provide temporal credit assignment, but do not assign credit \emph{counterfactually} to states or actions that \emph{could} have led to the current state, but did not do so this time. 
Credit will eventually trickle backwards over the course of multiple visits, but this can take many iterations.
As an example, suppose we collect a key to open a door, which leads to an unexpected reward.  Using standard one-step TD learning, we would update the state in which the door opened. Using eligibility traces, we would also update the preceding trajectory, including the acquisition of the key. But we would not update other sequences that \emph{could} have led to the reward, such as collecting a spare key or finding a different entrance.

\begin{figure}[t]
\centering
\includegraphics[width=1.\linewidth]{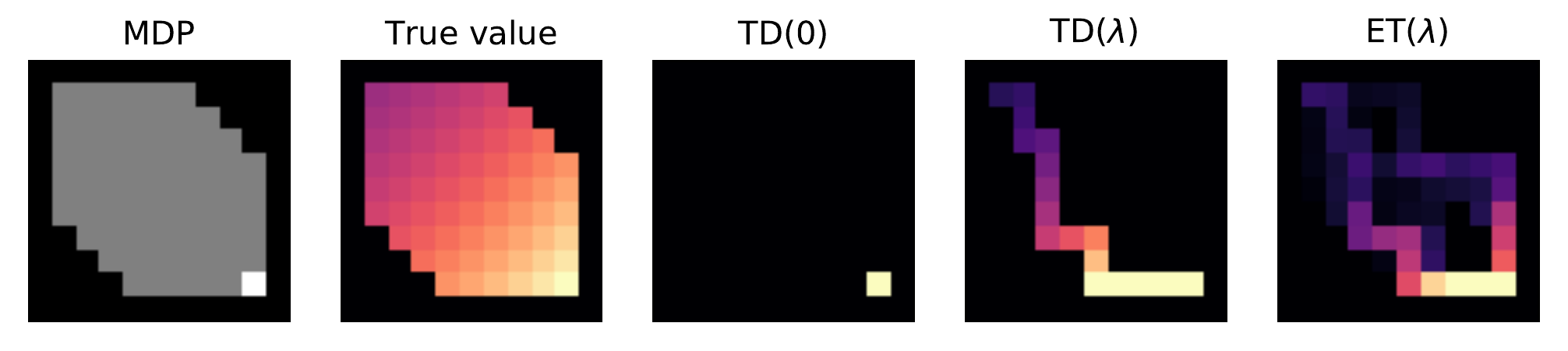}
\caption{\label{illustration1} A comparison of TD(0), TD($\lambda$), and the new expected-trace algorithm ET($\lambda$) (with $\lambda=0.9$). The MDP is illustrated on the left.  Each episode, the agent moves randomly down and right from the top left to the bottom right, where any action terminates the episode. Reward on termination are $+1$ with probability 0.2, and zero otherwise---all other rewards are zero. We plot the value estimates after the first positive reward, which occurred in episode 7. We see a) TD(0) only updated the last state, b) TD($\lambda$) updated the trajectory in this episode, and c) ET($\lambda$) additionally updated trajectories from earlier (unrewarding) episodes.}
\end{figure}

The problem of credit assignment to counterfactual states may be addressed by learning a model, and using the model to propagate credit \cite{Sutton:1990,Moore:93}; however, it has often proven challenging to construct and use models effectively in complex environments \citep[cf.][]{vanHasselt:2019}. Similarly, \emph{source traces} \citep{Pitis18} model full potential histories in tabular settings, but rely on estimated importance-sampling ratios of state distributions, which are hard to estimate in non-tabular settings.

We introduce a new approach to counterfactual credit assignment, based on the concept of \emph{expected eligibility traces}. We present a family of algorithms, which we call ET($\lambda$), that use expected traces to update their predictions. We analyse the nature of these expected traces, and illustrate their benefits empirically in several settings---see Figure \ref{illustration1} for a first illustration. We introduce a bootstrapping mechanism that provides a spectrum  of algorithms between standard eligibility traces and expected eligibility traces, and also discuss ways to apply these ideas with deep neural networks. Finally, we discuss possible extensions and connections to related ideas such as successor features.

\section{Background}
\label{background}
Sequential decision problems can be modelled as Markov decision processes\footnote{The ideas extend naturally to POMDPs \citep[cf.][]{Kaelbling:95}.} (MDP) $(\S, \A, p)$ \citep{Puterman:1994}, with state space $\S$, action space $\A$, and a joint transition and reward distribution $p(r, s' | s, a)$. An agent selects actions according to its policy $\pi$, such that $A_t \sim \pi(\cdot|S_t)$ where $\pi(a|s)$ denotes the probability of selecting $a$ in $s$, and observes random rewards and states generated according to the MDP, resulting in trajectories $\tau_{t:T} = \{S_t, A_t, R_{t+1}, S_{t+1}, \ldots, S_T\}$.  A central goal is to predict \emph{returns} of future discounted rewards \citep{SuttonBarto:2018}
\begin{align*}
G_t \equiv G(\tau_{t:T}) &= R_{t+1} + \g_{t+1} R_{t+2} + \g_{t+1} \g_{t+2}R_{t+3} + \ldots \\
&= \sum_{i=1}^{T} \g^{(i-1)}_{t+i} R_{t+i}\,,
\end{align*}
where $T$ is for instance the time the current episode terminates or $T=\infty$, and where $\g_t \in [0, 1]$ is a (possibly constant) discount factor and $\g^{(i)}_t = \prod_{k=1}^i \g_{t+k}$. The value $v_\pi(s) = \E{ G_t | S_t=s, \pi }$ of state $s$ is the expected return for a policy $\pi$.  Rather than writing the return as a random variable $G_t$, it will be convenient to instead write it as an explicit function $G(\tau)$ of the random trajectory $\tau$. Note that $G(\tau_{t:T}) = R_{t+1} + \g_{t+1} G(\tau_{t+1:T})$.

We approximate the value with a function $v_{\w}(s) \approx v_{\pi}(s)$.  This can for instance be a table---with a single separate entry $w[s]$ for each state---a linear function of some input features, or a non-linear function such as a neural network with parameters $\w$. The goal is to iteratively update $\w$ with
\[
\w_{t+1} = \w_t + \Delta\w_t
\]
such that $v_{\w}$ approaches the true $v_{\pi}$. Perhaps the simplest algorithm to do so is the Monte Carlo (MC) algorithm
\[
\Delta\w_t \equiv \alpha ( R_{t+1} + \gamma_{t+1} G(\tau_{t+1:T}) - v_{\w}(S_t) ) \nabla_{\w} v_{\w}(S_t) \,.
\]
Monte Carlo is effective, but has high variance, which can lead to slow learning.  TD learning \citep{Sutton:1988,SuttonBarto:2018} instead replaces the return with the current estimate of its expectation $v(S_{t+1}) \approx G(\tau_{t+1:T})$, yielding
\begin{align}\label{eq:td}
\Delta\w_t & \equiv  \alpha \delta_t \nabla_{\w} v_{\w}(S_t) \,,\\
\text{where }~~
\delta_t & \equiv R_{t+1} + \gamma_{t+1} v_{\w}(S_{t+1}) - v_{\w}(S_t) \,,\notag
\end{align}
where $\d_t$ is called the temporal-difference (TD) error.
We can interpolate between these extremes, for instance with $\lambda$-returns which smoothly mix values and sampled returns:
\[
G^\l(\tau_{t:T}) = R_{t+1} + \g_{t+1} \big((1 - \l) v_\w(S_{t+1}) + \l G^\l(\tau_{t+1:T})\big) \,.
\]
`Forward view' algorithms, like the MC algorithm, use returns that depend on future trajectories and need to wait until the end of an episode to construct their updates, which can take a long time.  Conversely, `backward view' algorithms rely only on past experiences and can update their predictions online, during an episode.  Such algorithms build an \emph{eligibility trace} \citep{Sutton:1988,SuttonBarto:2018}. An example is TD($\lambda$):
\begin{align*}
\Delta\w_t & \equiv \alpha \d_t \e_t \,,
& \text{with} &&
\e_t & = \g_t \l \e_{t-1} + \nabla_{\w} v_{\w}(S_t) \,,
\end{align*}
where $\e_t$ is an accumulating eligibility trace. This trace can be viewed as a function $\e_t\equiv \e(\tau_{0:t})$ of the trajectory of past transitions. The TD update in \eqref{eq:td} is known as TD(0), because it corresponds to using $\lambda=0$. TD($\lambda=1$) corresponds to an online implementation of the MC algorithm. Other variants exist, using other kinds of traces, and equivalences have been shown between these algorithms and their forward view using $\lambda$-returns: these backward-view algorithms converge to the same solution as the corresponding forward view, and can in some cases yield equivalent weight updates \citep{Sutton:1988,vanSeijen:2014,vanHasselt:2015}.

\section{Expected traces}
\label{algo}
The main idea is to use the concept of an \emph{expected eligibility trace}, defined as
\[
\z(s) \equiv \E{ \e_t \mid S_t=s } \,,
\]
where the expectation is over the agent's policy and the MDP dynamics.
We introduce a concrete family of algorithms, which we call ET($\lambda$) and ET($\lambda$, $\eta$), that learn expected traces and use them in value updates. We analyse these algorithms theoretically, describe specific instances, and discuss computational and algorithmic properties.

\subsection{ET($\lambda$)}
We propose to learn approximations $\z_{\th}(S_t) \approx \z(S_t)$, with parameters $\th \in \mathbb{R}^d$ (e.g., the weights of a neural network).  One way to learn $\z_{\th}$ is by updating it toward the instantaneous trace $\e_t$, by minimizing an empirical loss $\L(\e_t, \z_{\th}(S_t))$.
For instance, $\L$ could be a component-wise squared loss, optimized with stochastic gradient descent:
\begin{align*}
\th_{t+1} & = \th_t + \Delta\th_t \,,\text{ where}\\
\Delta\th_t
& = - \beta \frac{\partial}{\partial \th} \frac{1}{2} (\e_t - \z_{\th}(S_t))\tr (\e_t - \z_{\th}(S_t)) \\
&= \beta \frac{\partial \z_{\th}(S_t)}{\partial \th} (\e_t - \z_{\th}(S_t)) \,,
\end{align*}
where $\frac{\partial z_{\th}(S_t)}{\partial \th}$ is a $|\th| \times |\e|$ Jacobian\footnote{Auto-differentiation can efficiently compute this update with comparable computation to the loss calculation.} and $\beta$ is a step size.

\begin{figure}
\centering
\vspace{-0.8cm}
\begin{minipage}[t]{0.44\textwidth}
\begin{algorithm}[H]
    \caption{ET($\lambda$)}\label{tdlambda_exp}
    \begin{algorithmic}[1]
        \State \text{initialise $\w$, $\th$}
        \For{$M$ episodes}
        \State \text{initialise} $\e = \bm 0$
        \State \text{observe initial state $S$} 
        \Repeat{ for each step in episode $m$}
        \State \text{generate $R$ and $S'$} 
        \State $\d \gets R + \g v_{\w}(S') - v_{\w}(S)$ 
        \State $\e \gets \gamma \lambda \e + \nabla_{\w} v_{\w}(S)$ 
        \State $\th \gets \th + \beta \frac{\partial \z_{\th}(S)}{\partial \th} (\e - \z_{\th}(S))$
        \State $\w \gets \w + \a \delta \z_{\th}(S)$
        \Until{ $S$ \text{is terminal}}
        \EndFor
        \State \textbf{Return}  $\w$
    \end{algorithmic}
\end{algorithm}
\end{minipage}
\end{figure}
The idea is then to use $\z_{\th}(s) \approx \E{ \e_t \mid S_t=s }$ in place of $\e_t$ in the value update, which becomes
\begin{align}
\label{eq:et_update}
\Delta\w_t \equiv \d_t \z_{\th}(S_t) \,.
\end{align}
We call this ET($\lambda$).
Below, we prove that this update can be unbiased and can have lower variance than TD($\lambda$). Algorithm~\ref{tdlambda_exp} shows pseudo-code for a concrete instance of ET($\lambda$).

\subsection{Interpretation and ET($\lambda, \eta$)}
We can interpret TD(0) as taking the MC update and replacing the return from the subsequent state, which is a function of the future trajectory, with a state-based estimate of its expectation: $v(S_{t+1}) \approx \E{ G(\tau_{t+1:T}) | S_{t+1} }$. This becomes most clear when juxtaposing the updates
\begin{align*}
    \Delta\w_t & \equiv \alpha ( R_{t+1} + \gamma_{t+1} G(\tau_{t+1:T}) - v_{\w}(S_t) ) \bm{\nabla}_t \,,\tag{MC}\\
    \Delta\w_t & \equiv \alpha ( R_{t+1} + \gamma_{t+1} v_{\w}(S_{t+1}) - v_{\w}(S_t) ) \bm{\nabla}_t \,,\tag{TD}
\end{align*}
where we used a shorthand $\bm{\nabla}_t \equiv \nabla_{\w} v_{\w}(S_t)$.

TD($\lambda$) also uses a function of a trajectory: the trace $\e_t$.  We propose replacing this as well with a function state $\z_{\th}(S_t) \approx \E{ \e(\tau_{0:t}) | S_t }$: the expected trace. Again juxtaposing:
\begin{align*}
    \Delta\w_t & \equiv \alpha \d_t \e(\tau_{0:t}) \,,\tag{TD($\lambda$)}\\
    \Delta\w_t & \equiv \alpha \d_t \z_{\th}(S_t) \,.\tag{ET($\lambda$)}
\end{align*}

When switching from MC to TD(0), the dependence on the trajectory was replaced with a state-based value estimate to bootstrap on.  We can interpolate smoothly between MC and TD(0) via $\lambda$.  This is often useful to trade off variance of the return with potential bias of the value estimate.  For instance, we might not have access to the true state $s$, and might instead have to rely on features $\x(s)$. Then we cannot always represent or learn the true values $v(s)$---for instance different states may be aliased \citep{Whitehead:1991}.

Similarly, when moving from TD($\lambda$) to ET($\lambda$) we replaced a trajectory-based trace with a state-based estimate.  This might induce bias and, again, we can smoothly interpolate by using a recursively defined mixture trace $\ee_t$, as defined as\footnote{While $\ee_t$ depends on both $\eta$ and $\lambda$ we leave this dependence implicit, as is conventional for traces.}
\begin{equation}
\label{eq:mixture}    
\ee_t = (1 - \eta) \z_{\th}(S_t) + \eta \big( \g_t \l \ee_{t-1} + \nabla_{\w} v_{\w}(S_t) \big) \,.
\end{equation}
This recursive usage of the estimates $\z_{\th}(s)$ at previous states is analogous to bootstrapping on future state values when using a $\lambda$-return, with the important difference that the arrow of time is opposite.  This means we do not first have to convert this into a backward view: the quantity can already be computed from past experience directly. We call the algorithm that uses this mixture trace ET($\lambda$, $\eta$):
\begin{align*}
    \Delta\w_t & \equiv \alpha \d_t \ee(S_t) \,.\tag{ET($\lambda$, $\eta$)}
\end{align*}
Note that if $\eta=1$ then $\ee_t=\e_t$ equals the instantaneous trace: ET($\lambda$, $1$) is equivalent to TD($\lambda$). If $\eta=0$ then $\ee_t=\z_t$ equals the expected trace; the algorithm introduced earlier as ET($\lambda$) is equivalent to ET($\lambda$, $0$). By setting $\eta \in (0, 1)$, we can smoothly interpolate between these extremes.

\section{Theoretical analysis}\label{sec:analysis}

We now analyse the new ET algorithms theoretically. First we show that if we use $\z(s)$ directly and $s$ is Markov then the update has the same expectation as TD($\lambda$) (though possibly with lower variance), and therefore also inherits the same fixed point and convergence properties.

\begin{lemma}
\label{lmi}
If $s$ is Markov, then 
\[\E{ \d_t \e_t \mid S_t=s} = \E{ \d_t \mid S_t=s } \E{ \e_t \mid S_t=s }.\] 
\end{lemma}
\begin{proof} In Appendix~\ref{sec:proof_markov_ind}.
\end{proof}

\begin{restatable}{proposition}{PropMeanVar}
\label{prop:mean_and_variance}
\label{eqv}
Let $\e_t$ be any trace vector, updated in any way.  Let $\z(s) = \E{ \e_t \mid S_t=s }$.  Consider the ET($\lambda$) algorithm $\Delta\w_t = \alpha_t \d_t \z(S_t)$. For all Markov $s$ the expectation of this update is equal to the expected update with instantaneous trace $\e_t$, and the variance is lower or equal:
\begin{align*}
    \E{ \alpha_t \d_t \z(S_t) | S_t=s } & = \E{ \alpha_t \d_t \e_t | S_t=s}
    & \text{and}  \\
    \mathbb{V}[ \alpha_t \d_t \z(S_t) | S_t=s ] & \le \mathbb{V}[ \alpha_t \d_t \e_t | S_t=s ] \,,
\end{align*}
where the second inequality holds component-wise for the update vector, and is strict when $\mathbb{V}[\e_t  | S_t] > 0$.
\end{restatable}
\begin{proof}
We have
\begin{align}
    & \E{ \alpha_t \d_t \e_t \mid S_t=s} \notag\\
    & = \E{ \alpha_t \d_t \mid S_t=s } \E{ \e_t \mid S_t=s } \tag{Lemma~\ref{lmi}}\\
    & = \E{ \alpha_t \d_t \mid S_t=s } \z(s) \notag\\
    & = \E{ \alpha_t \d_t \z(S_t) \mid S_t=s }\,. \label{eq_means}
\end{align}
Denote the $i$-th component of $\z(S_t)$ by $z_{t,i}$ and the $i$-th component of $\e_t$ by $e_{t,i}$. Then, we also have
\begin{multline*}
    \E{ (\alpha_t \d_t z_{t,i})^2 | S_t=s}
    = \E{ \alpha_t^2 \d_t^2 \mid S_t=s } z_{t,i}^2 \hfill\\
    = \E{ \alpha_t^2 \d_t^2 \mid S_t=s } \E{ e_{t,i} | S_t=s }^2  \hfill\\
    = \E{ \alpha_t^2 \d_t^2 \mid S_t=s } \left(\E{ e_{t,i}^2 | S_t=s } - \mathbb{V}[{ e_{t,i} | S_t=s }] \right) \hfill\\ 
    \leq \E{ \alpha_t^2 \d_t^2 \mid S_t=s } \E{e_{t,i}^2 \mid S_t=s } \hfill \\
    = \E{ (\alpha_t \d_t e_{t,i})^2 \mid S_t=s } \,, \hfill
\end{multline*}
where the last step used the fact that $s$ is Markov, and the inequality is strict when $\mathbb{V}[\e_t  | S_t] > 0$. Since the expectations are equal, as shown in \eqref{eq_means}, the conclusion follows.
\end{proof}

\paragraph{Interpretation}
Proposition \ref{eqv} is a strong result: it holds for any trace update, including accumulating traces \citep{Sutton:84,Sutton:1988}, replacing traces \citep{Singh:96}, dutch traces \citep{vanSeijen:2014,vanHasselt:2014,vanHasselt:2015}, and future traces that may be discovered. It implies convergence of ET($\lambda$) under the same conditions as TD($\lambda$) \citep{Dayan:92,Peng:93,Tsitsiklis:94} with lower variance when $\mathbb{V}[\e_t  | S_t] > 0$, which is the common case.

Next, we consider what happens if we violate the assumptions of Proposition \ref{eqv}. We start by analysing the case of a learned approximation $\z_t(s) \approx \z(s)$ that relies solely on observed experience.

\begin{restatable}{proposition}{PropRunningMean}\label{prop:sample_mean_and_variance}
Let $\e_t$ an instantaneous trace vector.  Then let $\z_{t}(s)$ be the empirical mean $\z_{t}(s)= \frac{1}{n_t(s)} \sum_{i}^{n_t(s)} \e_{t_i^s}$, where $t_i^s$-s denote past times when we have been in state s, that is $S_{t_i^s}=s$, and $n_t(s)$ is the number of visits to $s$ in the first $t$ steps.  Consider the expected trace algorithm $\w_{t+1} = \w_t + \alpha_t \d_t \z_t$. If $S_t$ is Markov, the expectation of this update is equal to the expected update with instantaneous traces $\e_t$, while attaining a potentially lower variance:
\begin{align*}
    \E{ \alpha_t \d_t \z_t(S_t) \mid S_t }
    & = \E{ \alpha_t \d_t \e_t \mid S_t} &
     \text{and} \\
    \mathbb{V}{[ \alpha_t \d_t \z_t(S_t) \mid S_t]}
    & \leq  \mathbb{V}{[\alpha_t \d_t \e_t \mid S_t]} \,,
\end{align*}
where the second inequality holds component-wise. The inequality is strict when $\mathbb{V}[\e_t \mid S_t] > 0$.
\end{restatable}

\begin{proof} In Appendix.
\end{proof}

\paragraph{Interpretation}
Proposition~\ref{prop:sample_mean_and_variance} mirrors Proposition \ref{eqv} but, importantly, covers the case where we estimate the expected traces from data, rather than relying on exact estimates.  This means the benefits extend to this pure learning setting.  Again, the result holds for any trace update. The inequality is typically strict when the path leading to state $S_t=s$ is stochastic (due to environment or policy).

\def\X{\bm{\Theta}}
Next we consider what happens if we do not have Markov states and instead have to rely on, possibly non-Markovian, features $\x(s)$.  We then have to pick a function class and for the purpose of this analysis we consider linear expected traces $\z_{\X}(s) = \X \x(s)$ and values $v_{\w}(s) = \w\tr \x(s)$, as convergence for non-linear values can not always be assured even for standard TD($\lambda$) \citep{Tsitsiklis:1997}, without additional assumptions \citep[e.g.,][]{Ollivier:2018,Brandfonbrener:2020}.
The following property of the mixture trace is used in the proposition below.

\begin{restatable}{proposition}{PropMixtureTraceAsTD}
\label{prop:mixture_trace_rewrite_as_TD_trace}
The mixture trace $\ee_t$ defined in \eqref{eq:mixture} can be written as $\ee_t = \mu \ee_{t-1} + \u_t$ with decay parameter $\mu = \eta \gamma \lambda$ and signal $\u_t = (1 - \eta) \z_{\th}(S_t) + \eta \ \nabla_{\w} v_{\w}(S_t)$, such that
\begin{eqnarray}
\label{eq:mixture_trace_rewrite_as_TD_trace}
   \ee_t = \sum_{k=0}^t (\eta \g \l)^k  \left[ (1 - \eta) \z_{\th}(S_{t-k}) + \eta \ \nabla_{\w} v_{\w}(S_{t-k}) \right].
\end{eqnarray}
\end{restatable}

\begin{proof} In Appendix.
\end{proof}
Recall $\ee_t = \e_t$ when $\eta = 1$, and $\ee_t = \z_{\th}(S_t)$ when $\eta=0$, as can be verified by inspecting \eqref{eq:mixture_trace_rewrite_as_TD_trace} (and using the convention $0^0 = 1$). We use this proposition to prove the following.

\begin{restatable}{proposition}{PropETFixedPoint}
\label{prop:linear_ET_fixed_point}
When using approximations $z_{\X}(s) = \X \x(s)$ and $v_{\w}(s) = \w\tr \x(s)$ then, if $(1 - \eta) \X + \eta \mathbb{I}$ is non-singular, ET($\lambda$, $\eta$) has the same fixed point as TD($\lambda \eta$).
\end{restatable}
\begin{proof}
In Appendix.
\end{proof}

\paragraph{Interpretation}
This result implies that linear ET($\lambda$, $\eta$) converges under similar conditions as linear TD($\lambda'$) for $\lambda' = \lambda \cdot \eta$. In particular, when $\X$ is non-singular, using the approximation $\z_{\X}(s) = \X \x(s)$ in ET($\lambda$, 0) = ET($\lambda$) implies convergence to the fixed point of TD(0).

Though ET($\lambda$, $\eta$) and TD($\lambda \eta$) have the same fixed point, the algorithms are not equivalent. In general, their updates are not the same. Linear approximations are more general than tabular functions (which are linear functions of a indicator vector for the current state), and we have already seen in Figure \ref{illustration1} that ET($\lambda$) behaves quite differently from both TD(0) and TD($\lambda$), and we have seen its variance can be lower in Propositions \ref{eqv} and \ref{prop:sample_mean_and_variance}.
Interestingly, $\X$ resembles a preconditioner that speeds up the linear semi-gradient TD update, similar to how second-order optimisation algorithms \citep{Amari:1998,Martens:2016} precondition the gradient updates.

\section{Empirical analysis}

From the insights above, we expect that ET($\lambda$) yields lower prediction errors because it has lower variance and aggregates information across episodes better.  In this section we empirically investigate expected traces in several experiments. Whenever we refer to ET($\lambda$), this is equivalent to ET($\lambda$, $0$).

\begin{figure}[t]
\centering
\includegraphics[width=1.0\linewidth]{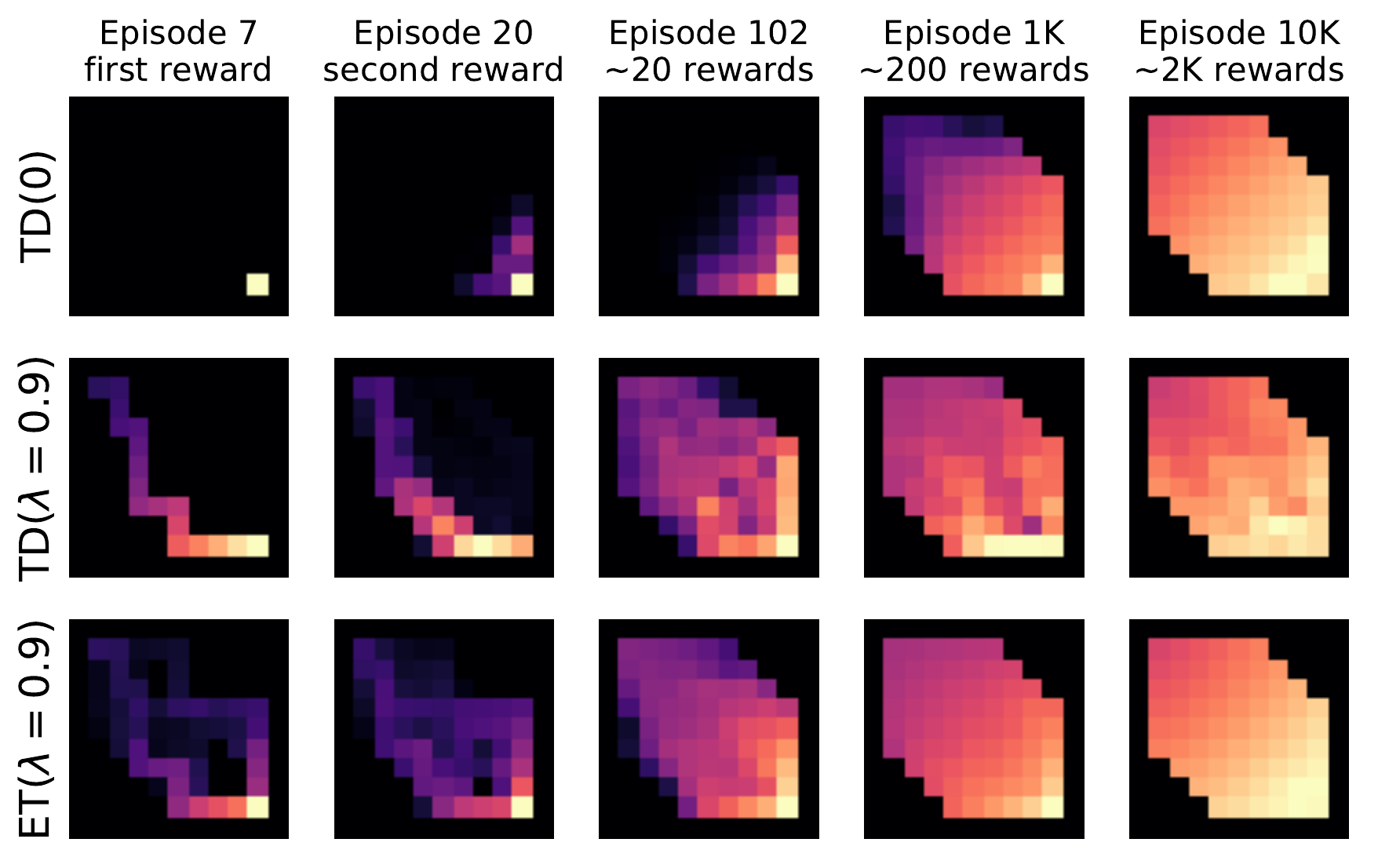}
\caption{\label{illustration2} In the same setting as Figure \ref{illustration1}, we show later value estimates after more rewards have been observed. TD(0) learns slowly but steadily, TD($\lambda$) learns faster but with higher variance, and ET($\lambda$) learns both fast and stable.}
\end{figure}
\subsection{An open world}
First consider the grid world depicted in Figure \ref{illustration1}. The agent randomly moves right or down (excluding moves that would hit a wall), starting from the top-left corner. Any action in the bottom-right corner terminates the episode with $+1$ reward with probability $0.2$, and $0$ otherwise. All other rewards are~$0$.

Figure \ref{illustration1} shows the value estimates after the first positive reward, which occurred in the seventh episode. TD(0) updated a single state, TD($\lambda$) updated earlier states in that episode, and ET($\lambda$) additionally updated states from previous episodes.
Figure \ref{illustration2} shows the values after the second reward, and after roughly $20$, $200$, and $2000$ rewards (or $100$, $1000$, and $10,000$ episodes, respectively).
ET($\lambda$) converged faster than TD(0), which propagated information slowly, and than TD($\lambda$), which had higher variance.  All step sizes decayed as $\alpha = \beta = \sqrt{1/k}$, where $k$ is the current episode number.

\subsection{A multi-chain}

\begin{figure}
\centering
\includegraphics[width=0.75\linewidth]{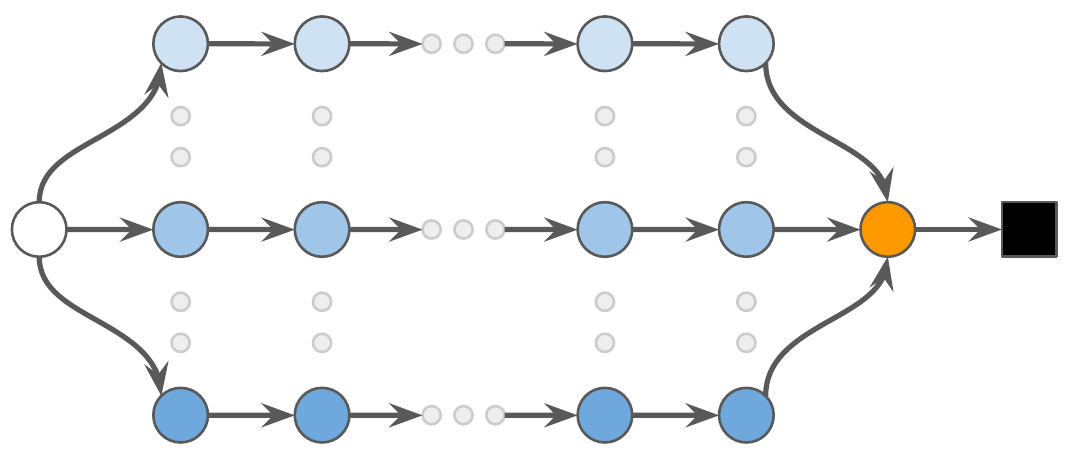}
\caption{\label{multi_chain} Multi-chain environment. Each episode starts in the left-most (white) state, and randomly transitions to one of $m$ parallel (blue) chains of identical length $n$. After $n$ steps, the agent always transitions to the same (orange) state, regardless of the chain it was in. The next step the episode terminates. Each reward is $+1$, except on termination when it either is $+1$ with probability $0.9$ or $-1$ with probability $0.1$.}
\end{figure}

\begin{figure*}[t]
\centering
\includegraphics[width=0.25\linewidth]{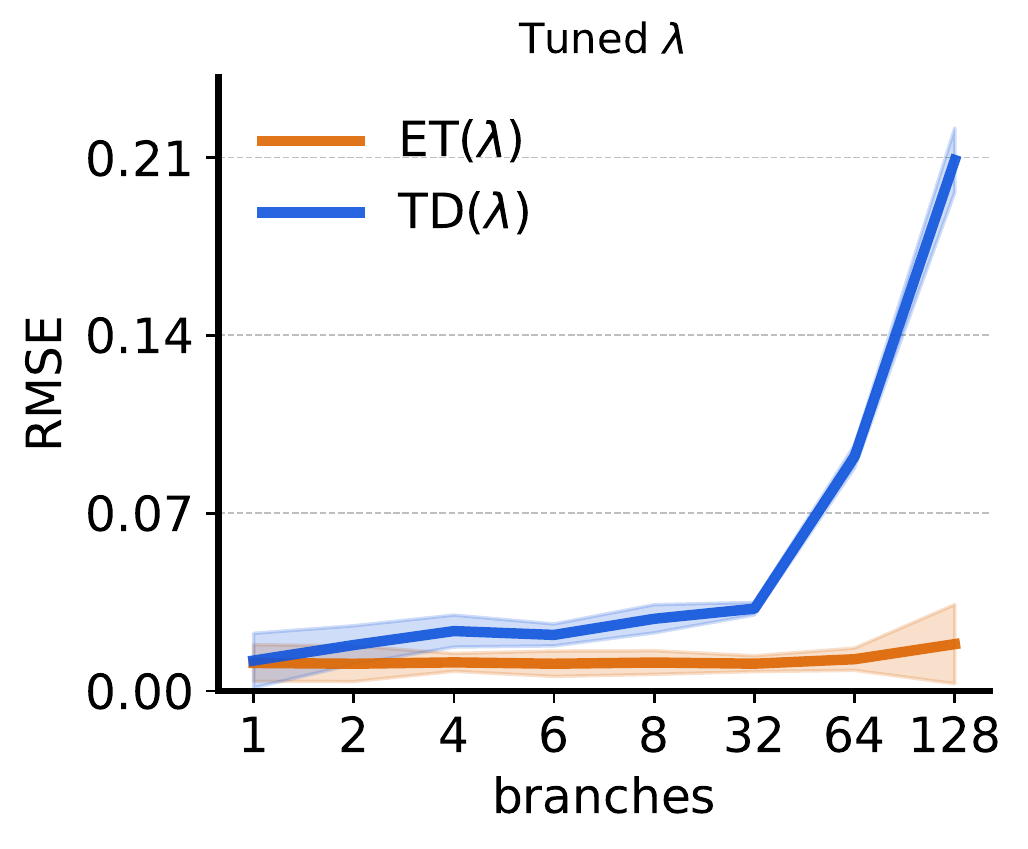}
\includegraphics[width=0.71\linewidth]{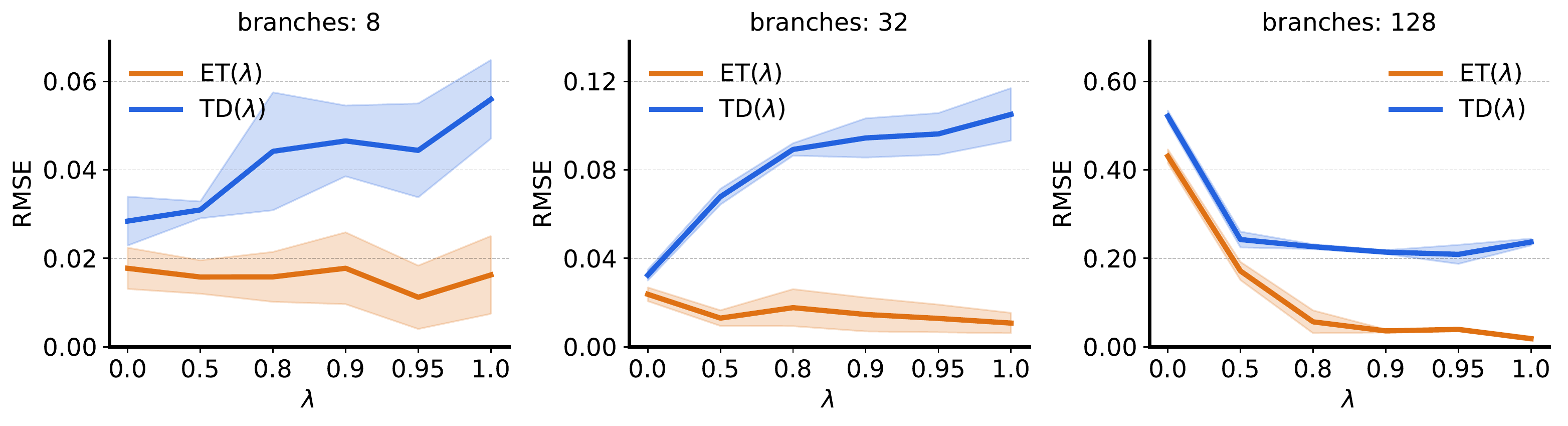}
\caption{\label{tabular} Prediction errors in the multi-chain. ET($\lambda$) (\orange{orange}) consistently outperformed TD($\lambda$) (\blue{blue}). Shaded areas depict standard errors across 10 seeds.}
\includegraphics[width=0.26\linewidth]{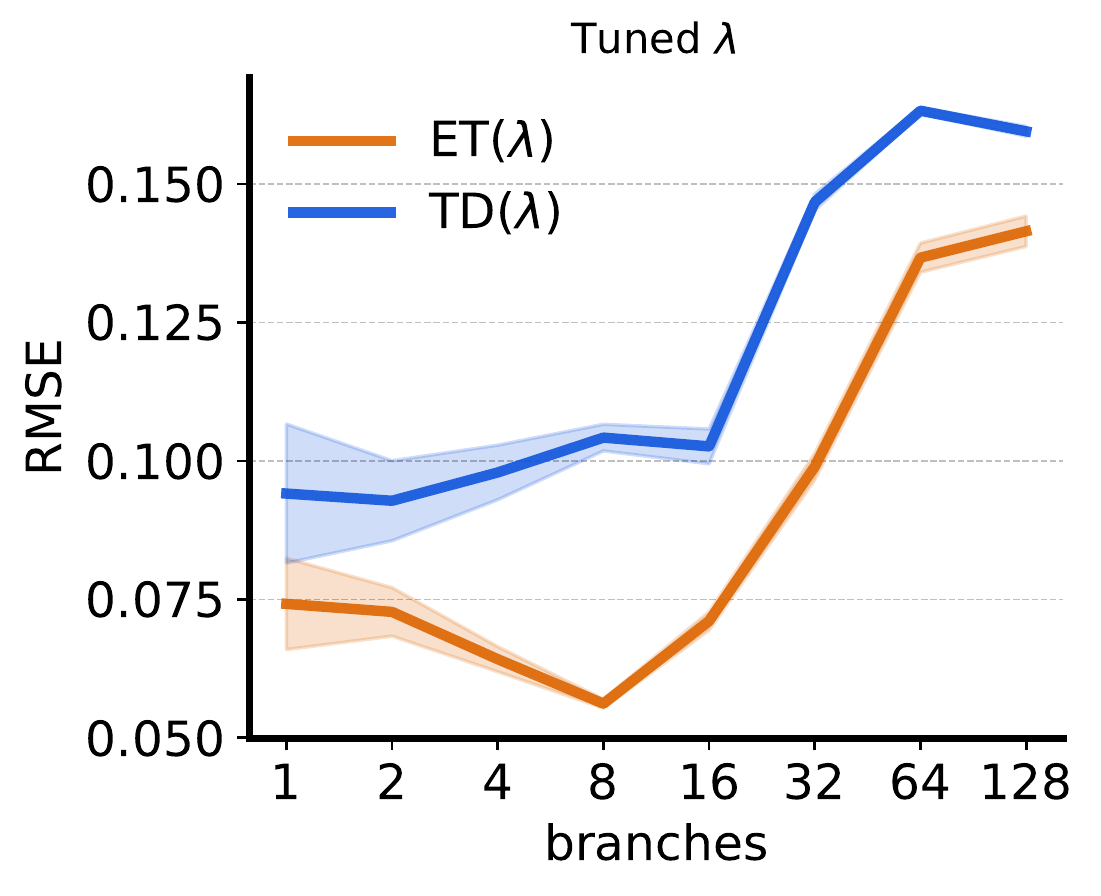}
\includegraphics[width=0.44\linewidth]{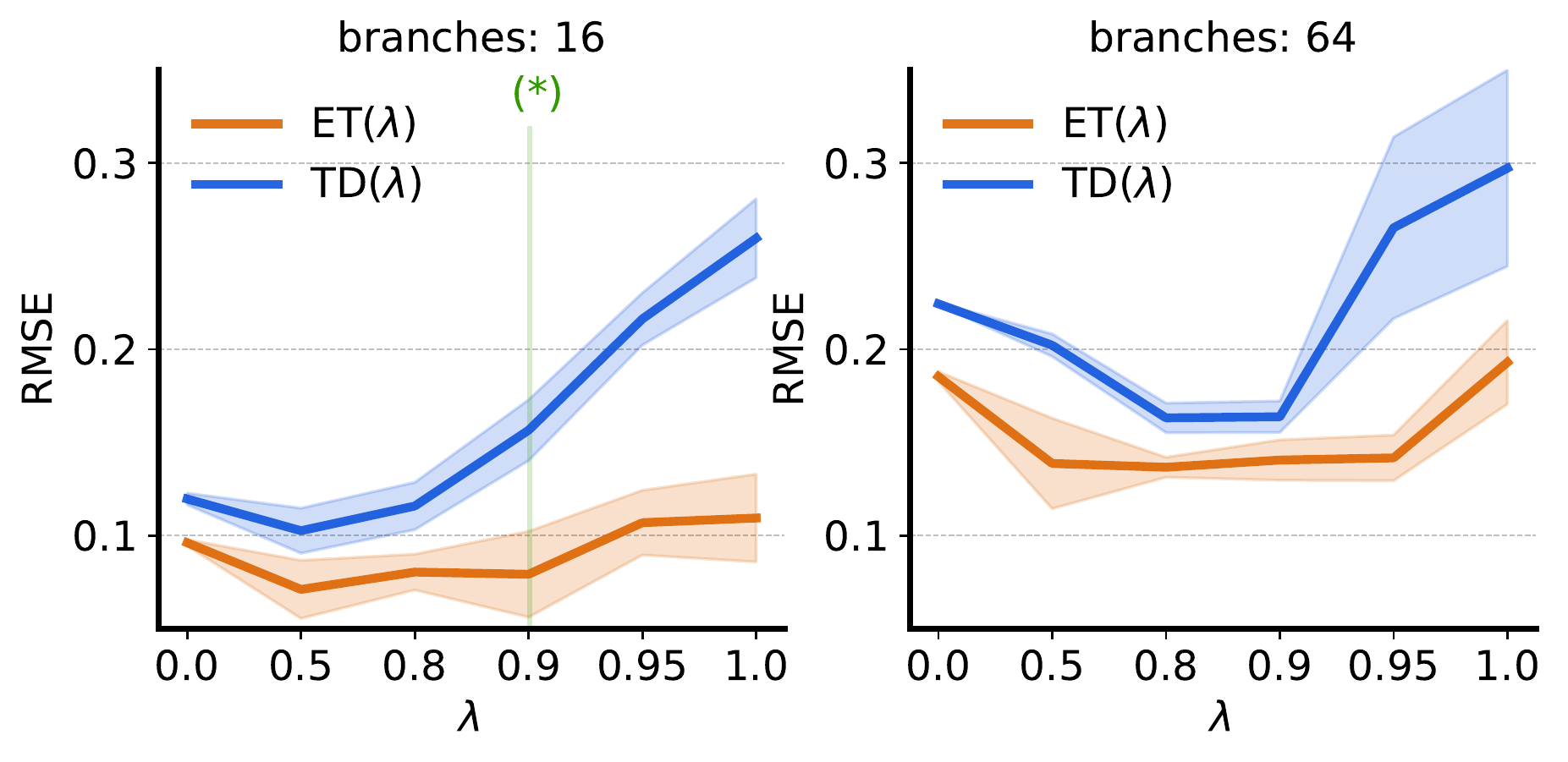}
\includegraphics[width=0.24\linewidth]{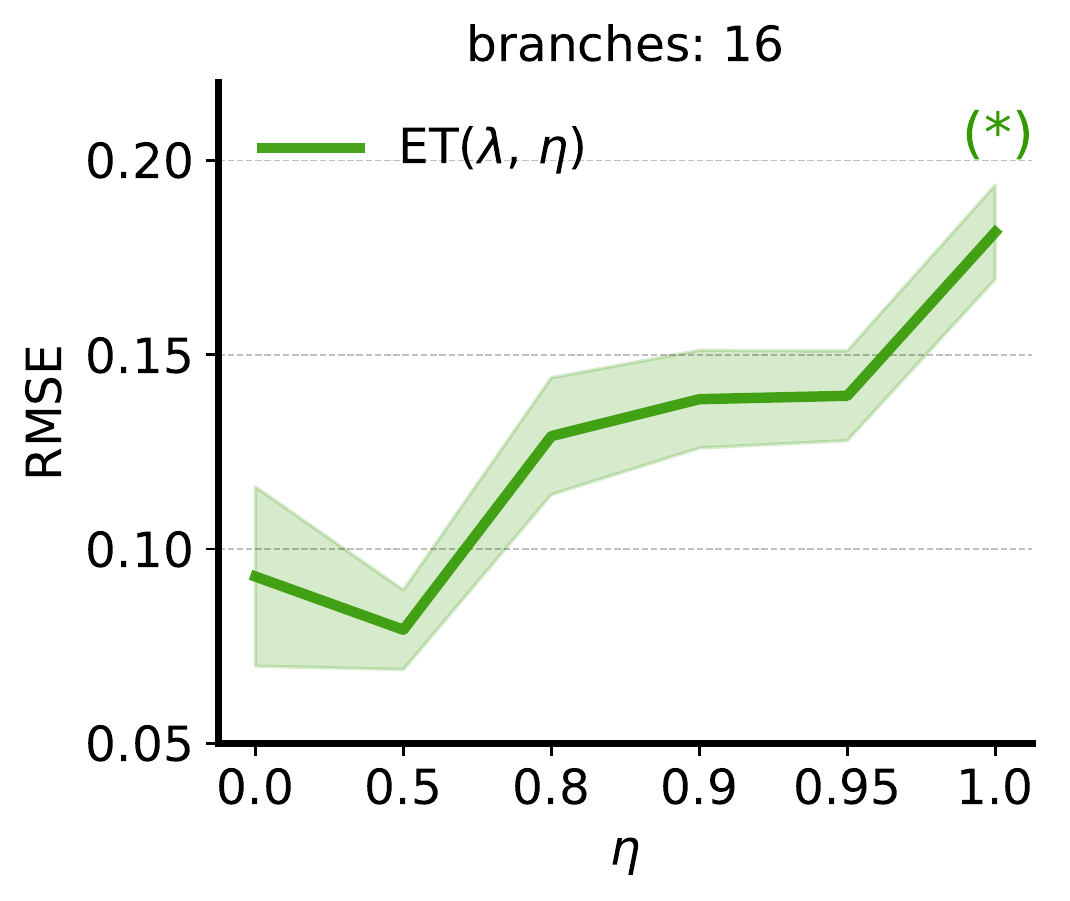}
\caption{\label{linear} Comparing value error with linear function approximation a) as function of the number of branches (left), b) as function of $\lambda$ (center two plots) and c) as function of $\eta$ (right). The left three plots show comparisons of TD($\lambda$) (\blue{blue}) and ET($\lambda$) (\orange{orange}), showing ET($\lambda$) attained lower prediction errors. The right plot interpolates between these algorithms via ET($\lambda$, $\eta$), from ET($\lambda$) = ET($\lambda$, $0$) to ET($\lambda$, $1$) = TD($\lambda$), with $\lambda=0.9$ (corresponding to a vertical slice indicated in the second plot).}
\end{figure*}
We now consider the multi-chain shown in Figure~\ref{multi_chain}. We first compare TD($\lambda$) and ET($\lambda$) with tabular values on various variants of the multi-chain, corresponding $n=4$ and $m \in \{1, 2, 4, 8, ..., 128\}$. The left-most plot in Figure \ref{tabular} shows the average root mean squared error (RMSE) of the value predictions after $1024$ episodes. We ran 10 seeds for each combination of step size $1 / t^d$ with $d \in \{0.5, 0.8, 0.9, 1\}$ and $\lambda \in \{0  , 0.5 , 0.8 , 0.9 , 0.95, 1 \}$.

The left plot in Figure \ref{tabular} shows value errors for different $m$, minimized over $d$ and $\lambda$. The prediction error of TD($\lambda$) (\blue{blue}) grew quickly with the number of parallel chains. ET($\lambda$) (\orange{orange}) scaled better, because it updates values in multiple chains (from past episodes) upon receiving a surprising reward (e.g., $-1$) on termination. The other three plots in Figure \ref{tabular} show value error as a function of $\lambda$ for a subset of problems corresponding to $m\in \{8, 32, 128\}$. The dependence on $\lambda$ differs across algorithms and problem instances; ET($\lambda$) always achieved lower error than TD($\lambda$). Further analysis, including on step-size sensitivity, is included in the appendix.

Next, we encode each state with a feature vector $\x(s)$ containing a binary indicator vector of the branch, a binary indicator of the progress along the chain, a bias that always equals one, and two binary features indicating when we are in the start (white) or bottleneck (orange) state. We extend the lengths of the chains to $n=16$. Both TD($\lambda$) and ET($\lambda$) use a linear value function $v_{\w}(s) = \w\tr \x(s)$, and ET($\lambda$) uses a linear expected trace $z_{\X}(s) = \X \x(s)$. All updates use the same constant step size $\alpha$. The left plot in Figure \ref{linear} shows the average root mean squared value error after $1024$ episodes (averaged over 10 seeds). For each point the best constant step size $\alpha \in \{0.01, 0.03, 0.1\}$ (shared across all updates) and $\lambda \in \{0  , 0.5 , 0.8 , 0.9 , 0.95, 1\}$ is selected. ET($\lambda$) (\orange{orange}) attained lower errors across all values of $m$ (left plot), and for all $\l$ (center two plots, for two specific $m$). The right plot shows results for smooth interpolations via $\eta$, for $\lambda=0.9$ and $m=16$. The full expected trace ($\eta=0$) performed well here, we expect in other settings the additional flexibility of $\eta$ could be beneficial.

\subsection{Expected traces in deep reinforcement learning}
\label{sec:deep}
\begin{figure*}[t]
\includegraphics[width=\textwidth]{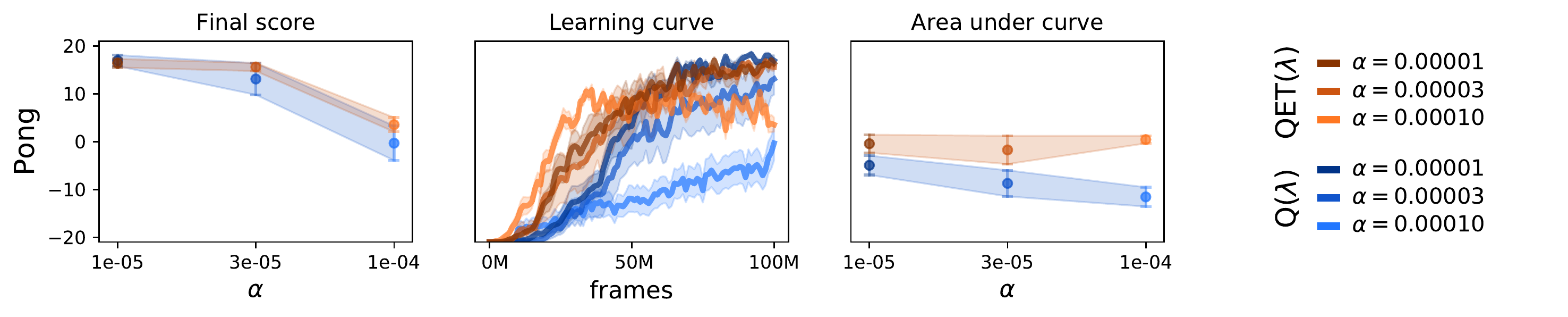}
\includegraphics[width=\textwidth]{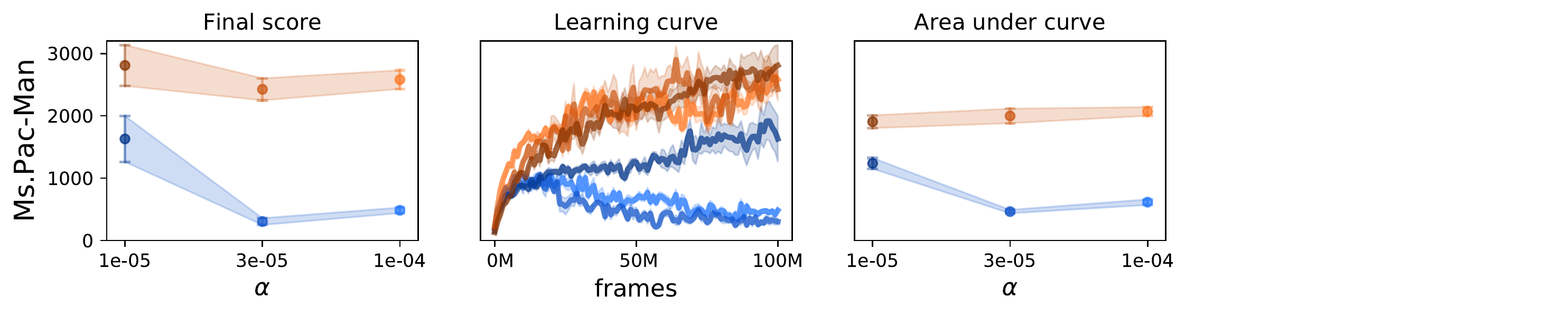}
\caption{\label{deep_et} Performance of Q($\lambda$) ($\eta=1$, \textbf{\textcolor{NavyBlue}{blue}}) and QET($\lambda$) ($\eta=0$, \textbf{\textcolor{Orange}{orange}}) on Pong and Ms.Pac-Man for various learning rates. Shaded regions show standard error across 10 random seeds. All results are for $\lambda=0.95$. Further implementation details and hyper-parameters are in the appendix.}
\end{figure*}

(Deep) neural networks are a common choice of function class in reinforcement learning \citep[e.g.,][]{Werbos:1990,Tesauro:92,Tesauro:94,Bertsekas:96,Prokhorov:1997,Riedmiller:2005,vanHasselt:2012,Mnih:2015,vanHasselt:2016,Wang:2016,Silver:2016,Duan:2016,Hessel:2018}. Eligibility traces are not very commonly combined with deep networks \citep[but see][]{Tesauro:92,Elfwing:2018}, perhaps in part because of the popularity of experience replay \citep{Lin:1992,Mnih:2015,Horgan:2018}.

Perhaps the simplest way to extend expected traces to deep neural networks is to first separate the value function into a representation $\x(s)$ and a value $v_{(\w, \bm{\xi})}(s) = \w\tr\x_{\bm{\xi}}(s)$, where $\x_{\bm{\xi}}$ is some (non-linear) function of the observations $s$.\footnote{Here $s$ denotes observations to the agent, not a full environment state---$s$ is not assumed to be Markovian.}  We can then apply the same expected trace algorithm as used in the previous sections by learning a separate linear function $\z_{\X}(s) = \X \x(s)$ using the representation which is learned by backpropagating the value updates:
\begin{align*}
    \bm{\xi}_{t+1} & = \bm{\xi}_t + \alpha \delta \e^{\bm{\xi}}_t
    \quad\text{ and }\quad 
    \w_{t+1} = \w_t + \alpha \delta \z_{\X}(S_t) \,,\\
\text{where }~~
    \e^{\bm{\xi}}_t & = \g_t \l \e^{\bm{\xi}}_{t-1} + \nabla_{\bm{\xi}} v_{(\w, \bm{\xi})}(S_t) \,,\\
    \e^{\w}_t & = \g_t \l \e^{\w}_{t-1} + \nabla_{\w} v_{(\w, \bm{\xi})}(S_t) \,,
\end{align*}
and then updating $\X$ by minimising the sum of component-wise squared differences between $\e^{\w}_t$ and $\z_{\X_t}(S_t)$.

Interesting challenges appear outside the fully linear case. First, the representation will itself be updated and will have its own trace $\e^{\bm{\xi}}_t$.  Second, in the control case we optimise behaviour: the policy will change. Both these properties of the non-linear control setting imply that the expected traces must track a non-stationary target. We found that being able to track this rather quickly improves performance: the expected trace parameters $\X$ in the following experiment were updated with a step size of $\beta=0.1$.

We tested this idea on two canonical Atari games: Pong and Ms. Pac-Man. The results in Figure \ref{deep_et} show that the expected traces helped speed up learning compared to the baseline which uses accumulating traces, for various step sizes.  Unlike most prior work on this domain, which often relies on replay \citep{Mnih:2015,Schaul:2016,Horgan:2018} or parallel streams of experience \citep{Mnih:2016}, these algorithms updated the values online from a single stream of experience. Further details are in the Appendix.

These experiments demonstrate that the idea of expected traces already extends to non-linear function approximation, such as deep neural networks. We consider this to be a rich area of further investigations. The results presented here are similar to earlier results \citep[e.g.,][]{Mnih:2015} and are not meant to compete with state-of-the-art performance results, which often depend on replay and much larger amounts of experience \citep[e.g.,][]{Horgan:2018}.

\section{Discussion and extensions}
We now discuss various interesting interpretations and relations, and discuss promising extensions.

\subsection{Predecessor features}
For linear value functions the expected trace $z(s)$ can be expressed non recursively as follows:
\begin{align}
\label{eq:pred_feat}
    \z(s)
    & = \E{ \sum_{n=0}^\infty \l^{(n)}_{t} \g^{(n)}_{t} \x_{t-n} \mid S_t = s} \,,
\end{align}

where $\g_k^{(n)} \equiv \prod_{j=k-n}^{k} \g_j$.
This is interestingly similar to the definition of the \textit{successor features} \citep{Barreto:2017}:
\begin{align}
\label{eq:succ_feat}
\psi(s) = \E{ \sum_{n=1}^\infty \g^{(n-1)}_{t} \x_{t+n} \mid S_t = s } \,.
\end{align}
The summation in \eqref{eq:succ_feat} is over future features, while in \eqref{eq:pred_feat} we have a sum over features already observed by the agent. 
We can thus think of linear expected traces as \textit{predecessor features}. A similar connection was made in the tabular setting by \citet{Pitis18}, relating source traces, which aim to estimate the source matrix $(I-\gamma P)^{-1}$, to successor representations \citep{dayan1993improving}. In a sense, the above generalises this insight. In addition to being interesting in its own right, this connection allows for an intriguing interpretation of $\z(s)$ as a multidimensional value function. Like with successor features, the features $\x_{t}$ play the role of rewards, discounted with $\gamma \cdot \lambda$ rather than $\gamma$, and with time flowing backwards.

Although the predecessor interpretation only holds in the linear case, it is also of interest as a means to obtain a practical implementation of expected traces with non-linear function approximation, for instance applied only to the linear `head' of a deep neural network. We used this `predecessor feature trick' in our Atari experiments described earlier.

\subsection{Relation to model-based reinforcement learning}

Model-based reinforcement learning provides an alternative approach to efficient credit assignment. The general idea is to construct a model that estimates state-transition dynamics, and to update the value function based upon hypothetical transitions drawn from the model \citep{Sutton:1990}, for example by prioritised sweeping \citep{Moore:93,vanSeijen:2013}. In practice, model-based approaches have proven challenging in environments (such as Atari games) with rich perceptual observations, compared to model-free approaches that more directly update the agent’s policy and predictions \citep{vanHasselt:2019}.

In some sense, expected traces also construct a model of the environment---but one that differs in several key regards from standard state-to-state models used in model-based reinforcement learning. First, expected traces estimate \emph{past} quantities rather than \emph{future} quantities. Second, they estimate the accumulation of gradients over a multi-step trajectory, rather than full transition dynamics, thereby focusing on those aspects that matter for the update. Third, they allow credit assignment across these potential past trajectories with a single update, without the iterative computation that is typically required when using a more explicit model.  These differences may be important to side-step some of the challenges faced in model-based learning.

\subsection{Batch learning and replay}
We have mainly considered the online learning setting in this paper.  It is often convenient to learn from batches of data, or replay transitions repeatedly, to enhance data efficiency. A natural extension is replay the experiences sequentially \citep[e.g.][]{Kapturowski:2018}, but perhaps alternatives exist.  We now discuss one potential extension.

We defined a mixed trace $\ee_t$ that mixes the instantaneous and expected traces.  Optionally the expected trace $\z_t$ can be updated towards the mixed trace $\ee_t$ as well, instead of towards the instantaneous trace $\e_t$. Analogously to TD($\lambda$) we propose to then use at least one real step of data:
\begin{equation}\label{mixed_z_update}
\Delta\th_t \equiv \beta \left(\bm{\nabla}_t + \g_t \l_t \ee_{t-1} - \z_{\th}(S_t)\right)\tr \frac{\partial \z_{\th}(S_t)}{\partial \th} \,,
\end{equation}
with $\bm{\nabla}_t \equiv \nabla_{\w} v_{\w}(S_t)$.
This is akin to a forward-view $\lambda$-return update, with $\nabla_{\w} v_{\w}(S_t)$ in the role of (vector) reward, and $\z_{\th}$ of value, and discounted by $\l_t\g_t$, but reversed in time. In other words, this can be considered a sampled Bellman equation \citep{Bellman:1957} but backward in time.

When we then choose $\eta=0$, then $\ee_{t-1} = z_{\th}(S_{t-1})$, and then the target in \eqref{mixed_z_update} only depends on a single transition. Interestingly, that means we can then learn expected traces from \emph{individual} transitions, sampled out of temporal order, for instance in batch settings or when using replay.

\subsection{Application to other traces}
We can apply the idea of expected trace to more traces than considered here.  We can for instance consider the characteristic eligibility trace used in REINFORCE \citep{Williams:1992} and related policy-gradient algorithms \citep{Sutton:2000}.

Another appealing application is to the follow-on trace or \emph{emphasis}, used in emphatic temporal difference learning \citep{Sutton:2016} and related algorithms \citep[e.g.,][]{Imani:2018}. Emphatic TD was proposed to correct an important issue with off-policy learning, which can be unstable and lead to diverging learning dynamics. Emphatic TD weights updates according to 1) the inherent interest in having accurate predictions in that state and, 2) the importance of predictions in that state for updating other predictions. Emphatic TD uses scalar `follow-on' traces to determine the `emphasis' for each update.  However, this follow-on trace can have very high, even infinite, variance.  Instead, we might estimate and use its expectation instead of the instantaneous emphasis. A related idea was explored by \citet{Zhang:2019} to obtain off-policy actor critic algorithms.

\section{Conclusion}
We have proposed a mechanism for efficient credit assignment, using the expectation of an eligibility trace.  We have demonstrated this can sometimes speed up credit assignment greatly, and have analyzed concrete algorithms theoretically and empirically to increase understanding of the concept.

Expected traces have several interpretations. First, we can interpret the algorithm as counterfactually updating multiple possible trajectories leading up to the current state. Second, they can be understood as trading off bias and variance, which can be done smoothly via a unifying $\eta$ parameter, between standard eligibility traces (low bias, high variance) and estimated traces (possibly higher bias, but lower variance). Furthermore, with tabular or linear function approximation we can interpret the resulting expected traces as predecessor states or features---object analogous to successor states or features, but time-reversed. Finally, we can interpret the linear algorithm as preconditioning the standard TD update, thereby potentially speeding up learning. These interpretations suggest that a variety of complementary ways to potentially extend these concepts and algorithms.

\bibliography{bibliography}
\appendix
\newpage
\onecolumn
\section{Appendix}

\subsection{Proof of Lemma~\ref{lmi} }
\label{sec:proof_markov_ind}

We start with a formal definition of a Markov state:

\begin{definition} 
\label{def:markov_state}
\textbf{Markov property}: we say a state $s$ is Markov if
\begin{equation*}
p(R_{t+1}, S_{t+1} | A_{t}, S_{t}, R_{t-1}, A_{t-1}, S_{t-1}...)
= p(R_{t+1}, S_{t+1} | A_{t}, S_{t}).
\end{equation*}
\end{definition}

\noindent
Next, we show that a Markov state implies a similar property for the transition probabilities induced by a policy $\pi$:

\begin{property} 
\label{prop:markov_state_policy}
Let $p_{\pi}$ be the transition probabilities induced by policy $\pi$. Then, if $s$ is Markov, we have that
\begin{equation*}
p_{\pi}(R_{t+1}, S_{t+1} | S_{t}, R_{t-1}, A_{t-1}, S_{t-1}...)
= p_{\pi}(R_{t+1}, S_{t+1} | S_{t}).
\end{equation*}
\end{property}
\begin{proof}
\begin{align}
\label{eq:def_pi_prob}
p_{\pi}(R_{t+1}, S_{t+1} | S_{t}, R_{t-1}, A_{t-1}, S_{t-1}...) 
& = \int_a \pi(a | S_t) p(S_{t+1}, R_{t+1} | A_{t} = a, S_{t}, R_{t-1}, A_{t-1}, S_{t-1}...) da \\   
& = \int_a \pi(a | S_t) p(S_{t+1}, R_{t+1} | A_{t} = a, S_{t}) da. \tag{Markov property}   
\end{align}
\end{proof}

\noindent
Using the above, we can prove Lemma~\ref{lmi}:

\setcounter{lemma}{0}
\begin{lemma}
If $s$ is Markov, then 
 $\E{ \d_t \e_t \mid S_t=s} = \E{ \d_t \mid S_t=s } \E{ \e_t \mid S_t=s }$.
\end{lemma}
\begin{proof}
First note that the expectations above are with respect to the transition probabilities $p_{\pi}$ as defined in ~(\ref{eq:def_pi_prob}). That noted, the result trivially follows from the fact that, when $s$ is Markov, the two random variables $\d_t$ and $\e_t$ are independent conditioned on $S_t$. To see why this is so, note that $\d_t$ is defined as
\begin{align}
    \d_t = R_{t+1} + \gamma_{t+1} v_{\w}(S_{t+1}) - v_{\w}(S_t).
\end{align}
Since we are conditioning on the event that $S_t = s$, the only two random quantities in the definition of $\d_t$ are $R_{t+1}$ and $S_{t+1}$. Thus, because $s$ is Markov, we have that 
\begin{equation*}
p_{\pi}(\delta_t | S_t, R_{t}, S_{t-1}, R_{t-1}, ...) = p_{\pi}(\delta_t | S_t),  \tag{Property~\ref{prop:markov_state_policy}}
\end{equation*}
that is, $S_{t}$ fully defines the distribution of $\d_t$. This means that $p_{\pi}(\delta_t | S_t, X_{t'}) = p_{\pi}(\delta_t | S_t)$ for any $t' \le t$, where $X_{t'}$ is a random variable that only depends on events that occurred up to time $t'$. Replacing $X_{t'}$ with $\e_t$, we have that $p_{\pi}(\d_t | S_t, \e_t) = p_{\pi}(\d_t | S_t)$, which implies that $\d_t$ and $\e_t$ are independent conditioned on $S_t$.
\end{proof}

\begin{comment}
\LemMarkovInd
\begin{proof}
First note that the expectations above are with respect to the transition probabilities $p_{\pi}$ as defined in ~(\ref{eq:def_pi_prob}). That noted, the result trivially follows from the fact that, when $s$ is Markov, the two random variables $\d_t$ and $\e_t$ are independent conditioned on $S_t$. To see why this is so, note that $\d_t$ is defined as
\begin{align}
    \d_t = R_{t+1} + \gamma_{t+1} v_{\w}(S_{t+1}) - v_{\w}(S_t).
\end{align}
The only two random quantities in the definition of $\d_t$ are $R_{t+1}$ and $S_{t+1}$. Because $s$ is Markov, we have that 
\begin{equation*}
p_{\pi}(R_{t+1}, S_{t+1} | S_t, R_{t}, S_{t-1}, R_{t-1}, ...) = p_{\pi}(R_{t+1}, S_{t+1} | S_t),  \tag{Property~\ref{prop:markov_state_policy}}
\end{equation*}
that is, $S_{t}$ fully defines the distribution of $R_{t+1}$ and $S_{t+1}$, and thus of $\d_t$. This means that $p_{\pi}(\d_t | S_t, \e_t) = p_{\pi}(\d_t | S_t)$, which in turn implies that $\d_t$ and $\e_t$ are independent conditioned on $S_t$.
\end{proof}

\end{comment}

\subsection{Proof of Proposition~\ref{prop:sample_mean_and_variance} }
\label{sec:proof_sample_and_variance}
\PropRunningMean*
\begin{proof}
We have
\begin{align*}
    \E{ \alpha_t \d_t \e_t \mid S_t=s}
    & = \E{ \alpha_t \d_t \mid S_t=s } \E{ \e_t \mid S_t=s } \tag{as $s$ is Markov}\\
    & = \E{ \alpha_t \d_t \mid S_t=s } \E{ \z_t \mid S_t=s } \tag{as $\z_t = \frac{1}{n} \sum_{i}^{n} \e_{t_i^s} $} \\
    & = \E{ \alpha_t \d_t \z_t \mid S_t=s }. \hspace{20em}\qedsymbol
\end{align*}

Now let us look at the conditional variance for each of the dimension of the update vector $\alpha_t \d_t \z_t$: $\mathbb{V}[\alpha_t \d_t \z_{t, i} \mid S_t=s  ]$, where $\z_{t, i}$ denotes the $i$-th component of vector $\z_{t}$.
\begin{align*}
        & \mathbb{V}[\alpha_t \d_t \z_{t, i} \mid S_t=s  ] \\
        &= \E{(\alpha_t \d_t \z_{t, i} )^2\mid S_t=s} - \E{\alpha_t \d_t \z_{t, i} \mid S_t=s}^2 \\
        &= \E{\alpha_t^2 \d_t^2 (\z_{t, i} )^2\mid S_t=s} - \E{\alpha_t \d_t\mid S_t=s}^2 \E{\z_{t, i} \mid S_t=s}^2 \\
        &= \E{\alpha_t^2 \d_t^2 \mid S_t=s} \E{ (\z_{t, i} )^2\mid S_t=s} - \E{\alpha_t \d_t\mid S_t=s}^2 \E{\z_{t, i} \mid S_t=s}^2
\end{align*}
By a similar argument, we have
\begin{align*}
        & \mathbb{V}[\alpha_t \d_t \e_{t, i} \mid S_t=s  ]  \\
        &= \E{\alpha_t^2 \d_t^2 \mid S_t=s} \E{ (\e_{t, i} )^2\mid S_t=s} - \E{\alpha_t \d_t\mid S_t=s}^2 \E{\e_{t, i} \mid S_t=s}^2
\end{align*}
Now, we also know that $\E{\z_t \mid S_t=s} = \E{\e_t \mid S_t=s} = \mu_t$, as $\z_t$ is the empirical mean of $\e_t$. Thus we also have, component-wise,
\[
\E{\z_{t,i} \mid S_t=s} = \E{\e_{t,i} \mid S_t=s} = \mu_{t,i}
\]
Moreover, from the same reason we have that $\mathbb{V}(\z_{t,i}|S_t=s) = \frac{1}{n_s} \mathbb{V}(\e_{t,i}|S_t=s)$. Thus we obtain:
\begin{eqnarray*}
        \mathbb{V}[\alpha_t \d_t \z_{t,i} \mid S_t=s  ] 
        &=& \E{\alpha_t^2 \d_t^2 \mid S_t=s} \E{ \z_{t,i} (\z_{t,i} )^T\mid S_t=s} - \E{\alpha_t \d_t\mid S_t=s}^2 \mu_{t,i}^2
\end{eqnarray*}
Thus:
\begin{align*}
        &\mathbb{V}[\alpha_t \d_t \z_{t,i} \mid S_t=s  ] - \mathbb{V}[\alpha_t \d_t \e_{t,i} \mid S_t=s  ] \\
        &=  \E{\alpha_t^2 \d_t^2 \mid S_t=s} \underbrace{\left(\E{ \z_{t, i} (\z_{t, i} )^T\mid S_t=s} - \E{ \e_{t, i} (\e_{t, i} )^T\mid S_t=s} \right)}_{\leq 0, \text{ from definition of $\z_{t, i}$}}
        ~~\leq~~  0 \,,
\end{align*}
with equality holding, if and only if:
\begin{enumerate}[i]
    \item $\E{ (\z_{t, i} )^2\mid S_t=s} = \E{  (\e_{t, i} )^2\mid S_t=s} \Rightarrow \mathbb{V}(\z_{t,i}|S_t=s) = \mathbb{V}(\e_{t,i}|S_t=s)$, but $\mathbb{V}(\z_{t,i}|S_t=s) = \frac{1}{n_s} \mathbb{V}(\e_{t,i}|S_t=s)$ by definition of $\z_{t,i}$ as the running mean on samples $\e_{t,i}$. This can only happen for $n_s=1$, or in the absence of stochasticity, for every state $s$. Thus, in the most general case, this implies $\mathbb{V}(\z_{t,i}|S_t=s)=\mathbb{V}(\e_{t,i}|S_t=s) = 0$; or
    \item $\E{\alpha_t^2 \d_t^2 \mid S_t=s} = 0 \Rightarrow \d_t = 0$
\end{enumerate}
Thus, we have equality only with we have exactly one sample for the average $\z_t$ so far, or only one sample is needed (thus $\z_t$ and $\e_t$ are not actual random variables and there is only one deterministic path to $s$); or when the TD errors are zero for all transitions following $s$.
 \end{proof}

 \subsection{Properties of mixture traces}
 \label{sec:properties_of_mixture_traces}
 In this section we explore and proof some of the properties of the proposed mixture trace, defined in Equation \eqref{eq:mixture} in the main text and repeated here:
 \begin{equation}
\ee_t = (1 - \eta) \z_{\th}(S_t) + \eta \big( \g_t \l_t \ee_{t-1} + \nabla_{\w} v_{\w}(S_t) \big) \,. \tag{\ref{eq:mixture}}
\end{equation}
The proofs, in this section we will use the notation $\x_{t}$ to denote the features used in a linear approximation for the value function(s) constructed. Just note that this term can be substituted, in general, by the gradient term $\nabla_{\w} v_{\w}(S_t)$ in the equation above.

\PropMixtureTraceAsTD*

\begin{proof}
As mentioned before, under a linear parameterization $ \nabla_{\w} v_{\w}(S_t) = \x(S_t) := \x_{t}$
Let us start with the definition of the mixture trace $\ee_t$:
\begin{align*}
  \ee_t &= (1 - \eta) \z_t + \eta (\g_t \l_t \ee_{t-1} + \x_t)  \notag\\
  &= \left[ (1 - \eta) \z_t + \eta \x_t \right] + \eta \g_t \l_t \ee_{t-1}  \notag\\
  &= \left[ (1 - \eta) \z_t + \eta \x_t \right] + \eta \g_t \l_t  \left[ (1 - \eta) \z_{t-1} + \eta \x_{t-1} \right]
  + \eta^2 \g_t \l_t \g_{t-1} \l_{t-1} \ee_{t-2}  \notag\\
  &= (1 - \eta) \left[ \z_t + \eta \g_t \l_t \z_{t-1 } + \eta^2 \g_t \l_t \g_{t-1} \l_{t-1}  \z_{t-2} + \cdots \right]+  \notag\\
  &\phantom{=} + \eta \left[ \x_t + \eta \g_t \l_t \x_{t-1 } + \eta^2 \g_t \l_t \g_{t-1} \l_{t-1}  \x_{t-2} + \cdots \right]  \notag\\
    &= (1 - \eta) \sum_{k=0}^t (\eta \g \l)^k  \z_{t-k} + \eta \sum_{k=0}^t (\eta \g \l)^k  \x_{t-k} \label{eq:eta_lambda_convex_combination_interpretation_appendix}\\
    &= \sum_{k=0}^t (\eta \g \l)^k  \left[ (1 - \eta) \z_{t-k} + \eta \x_{t-k} \right] \notag
\end{align*}
Substituting $\x_{t}$ in the above derivation by $\nabla_{\w} v_{\w}(S_t)$ leads to \eqref{eq:mixture}.
\end{proof}

\PropETFixedPoint*
 \label{sec:linear_ET_fixed_point}
 \begin{proof}
By Proposition \ref{prop:mixture_trace_rewrite_as_TD_trace} we have that $\ee_t$ can be re-written as:
\begin{align}
\ee_t &= \sum_{k=0}^t (\eta \g \l)^k  \left[ (1 - \eta) \z_{\th}(S_{t-k}) + \eta \x(S_{t-k}) \right] \notag\\
 &= \sum_{k=0}^t (\eta \g \l)^k  \left[ (1 - \eta) \X \x(S_{t-k}) + \eta \x(S_{t-k}) \right] \\
 &= \left[ (1 - \eta) \X + \eta \mathbb{I} \right] \underbrace{\sum_{k=0}^t (\eta \g \l)^k  \x(S_{t-k})}_{\text{instantaneous trace $\e^{\lambda \eta}_t$}} \,. \label{eq:e_eta_theta}
\end{align}
We examine the fixed point $\w_*$ of the algorithm using this approximation of the expected trace:
\begin{align*}
\E{\delta_t \ee_t}
& = \E{\ee_t (R_{t+1} + \gamma \x(S_{t+1})\tr \w_* - \x(S_t)\tr \w_*)} \\
& = \mathbf{0} \,.
\end{align*}
This implies the fixed point is
\[
\w_* = \E{\ee_t (\gamma \x(S_{t+1}) - \x(S_t))\tr}^{-1} \E{\ee_t R_{t+1}} \,.
\]
Now, plugging in the relation in \eqref{eq:e_eta_theta} above, we get:
\begin{align*}
\w_*
& = \E{ \left[ (1 - \eta) \X + \eta \mathbb{I} \right] \e^{\lambda \eta}_t (\gamma \x(S_{t+1}) - \x(S_t))\tr}^{-1} \E{\left[ (1 - \eta) \X + \eta \mathbb{I} \right] \e^{\lambda \eta}_t R_{t+1}} \\
& = \E{\e^{\lambda \eta}_t (\gamma \x(S_{t+1}) - \x(S_t))\tr}^{-1} \left[ (1 - \eta) \X + \eta \mathbb{I} \right]^{-1} \left[ (1 - \eta) \X + \eta \mathbb{I} \right] \E{\e^{\lambda \eta}_t R_{t+1}} \\
& = \E{\e^{\lambda \eta}_t (\gamma \x(S_{t+1}) - \x(S_t))\tr}^{-1} \E{\e^{\lambda\eta}_t R_{t+1}} \,.
\end{align*}
This last term is the fixed point for TD($\lambda \eta$).
\end{proof}

Moreover, it is worth noting that the above equality recovers, for the extreme values of $\eta$:
\begin{itemize}
    \item $\eta=1 \Rightarrow \ee_t = \sum_{k=0}^t (\g \l)^k  \x_{t-k}$ (instantaneous trace for TD($\lambda$))
    \item $\eta=0 \Rightarrow \ee_t = \sum_{k=0}^t (\eta \g \l)^k  \z_{t-k} = \z_t$ (expected trace for TD($\lambda$))
\end{itemize}

Moreover, as the extreme values already suggest, the expected update of the mixture traces follows the TD($\lambda$) learning, in expectation, for all the intermediate values $\eta \in (0,1)$ as well, trading off variance of estimates as $\eta$ approaches $0$.

\begin{proposition}
Let $\e_t^{\lambda}$ be a $\lambda$ trace vector.  Let $\ee_t = (1 - \eta) \z_t + \eta (\g \l \ee_{t-1} + \x_t)$ (as defined in \eqref{eq:mixture}).  Consider the ET($\lambda, \eta$) algorithm $\w_{t+1} = \w_t + \alpha_t \d_t \ee_t$. For all Markov states $s$ the expectation of this update is equal to the expected update with instantaneous traces $\e_t^{\lambda}$:
\begin{align*}
    \E{ \alpha_t \d_t \ee(S_t) | S_t=s } & = \E{ \alpha_t \d_t \e_t^{\lambda} | S_t=s} \,,
\end{align*}
for every $\eta \in [0,1]$ and any $\lambda \in [0,1]$.
\end{proposition}

\begin{proof}
{\small
Let us revisit Eq.~\ref{eq:mixture_trace_rewrite_as_TD_trace} in Proposition \ref{prop:mixture_trace_rewrite_as_TD_trace}:
\begin{eqnarray*}
   \E{\ee_t}
    &=& \E{\sum_{k=0}^t (\eta \g \l)^k  \left[ (1 - \eta) \z_{t-k} + \eta \x_{t-k} \right]} \\
    &=& \E{\sum_{k=0}^t (\eta \g \l)^k  \left[ (1 - \eta) \E{\left( \x_{t-k} + \g \l \z_{t-k-1} \right)} + \eta \x_{t-k} \right]} \\
    &=& \E{\sum_{k=0}^t (\eta \g \l)^k  \x_{t-k} + (1 - \eta) \g \l \sum_{k=0}^{t-1} (\eta \g \l)^k \underbrace{\z_{t-k-1}}_{\E{(\x_{t-k-1} + \g \l \z_{t-k-2})}}}  \\
    &=& \E{\sum_{k=0}^t (\eta \g \l)^k  \x_{t-k} + (1 - \eta) \g \l \sum_{k=0}^{t-1} (\eta \g \l)^k \x_{t-k-1} + (1 - \eta) (\g \l)^2 \sum_{k=0}^{t-2} (\eta \g \l)^k \underbrace{\z_{t-k-2}}_{\E{\x_{t-k-2} + \g \l \z_{t-k-3}}}} \\
    &=& \E{\sum_{k=0}^t (\eta \g \l)^k  \x_{t-k} + (1 - \eta) \sum_{i=1}^{t-1}{(\g \l)}^i \sum_{k=0}^{t-i} (\eta \g \l)^k \x_{t-k-i}}
\end{eqnarray*}
Now, re-writing the sum, gathering all the weighting for each feature $\x_{t-k-i}$ we get:
\begin{eqnarray*}
   \E{\ee_t}
    &=& \E{\sum_{k=0}^t (\eta \g \l)^k  \x_{t-k} + (1 - \eta) \sum_{i=1}^{t-1}{(\g \l)}^i \sum_{k=0}^{t-i} (\eta \g \l)^k \x_{t-k-i}} \\
    &=& \E{\x_{t} + \sum_{k=1}^t \x_{t-k} \left((\eta \g \l)^k + (1 - \eta) \sum_{i=1}^{k}{(\g \l)^{i} \cdot(\g \l \eta)}^{k-i}\right)} \\
    &=& \E{\x_{t} + \sum_{k=1}^t \x_{t-k} (\g \l)^k \left(\eta^k + (1 - \eta) \sum_{i=1}^{k}{ \eta^{k-i}}\right)} \\
    &=& \E{\x_{t} + \sum_{k=1}^t \x_{t-k} (\g \l)^k \left(\eta^k + (1 - \eta) \frac{1-\eta^{k}}{(1-\eta)}\right)} \\
    &=& \E{\x_{t} + \sum_{k=1}^t \x_{t-k} (\g \l)^k} \\
    &=& \E{\sum_{k=0}^t (\g \l)^k \x_{t-k}}     
\end{eqnarray*}

Thus $\E{\ee_t} = \E{\sum_{k=0}^t (\g \l)^k \x_{t-k}} = \E{\e^{\lambda}_t }$, where $\e^{\lambda}_t$ is the instantaneous $\lambda$ trace on feature space $\x$. Thus $\E{\ee(s)} = z_{*}^{\lambda}(s) = \E{\e^{\lambda}_t }$.
Finally we can plug-in this result in the expected update:
\begin{align*}
    \E{ \alpha_t \d_t \ee(S_t) | S_t=s } & = \E{ \alpha_t \d_t | S_t=s} \E{ \ee(S_t) | S_t=s }  \\
    & = \E{ \alpha_t \d_t | S_t=s} z_{*}^{\lambda}(s) \\
    & = \E{ \alpha_t \d_t | S_t=s} \E{\e^{\lambda}_t | S_t=s } \\
    & = \E{ \alpha_t \d_t \e^{\lambda}_t | S_t=s }.
\end{align*}
}
\end{proof}
Finally, please note that in this proposition and its proof we drop the time indices $t$ for $\lambda$ and $\gamma$ parameters in the definition of $\ee_t$. This is purely to ease the notation and promote compactness in the derivation

\subsection{Parameter study}
Figure \ref{tabular_step_size} contains a parameter study in which we compare the performance of TD($\lambda$) and ET($\lambda$) across different step sizes.  The data used for this figure is the same as used to generate the plots in Figure \ref{tabular}, but now we look explicitly at the effect of the step size parameter.

\begin{figure}
    \centering
    \includegraphics[width=0.5\linewidth]{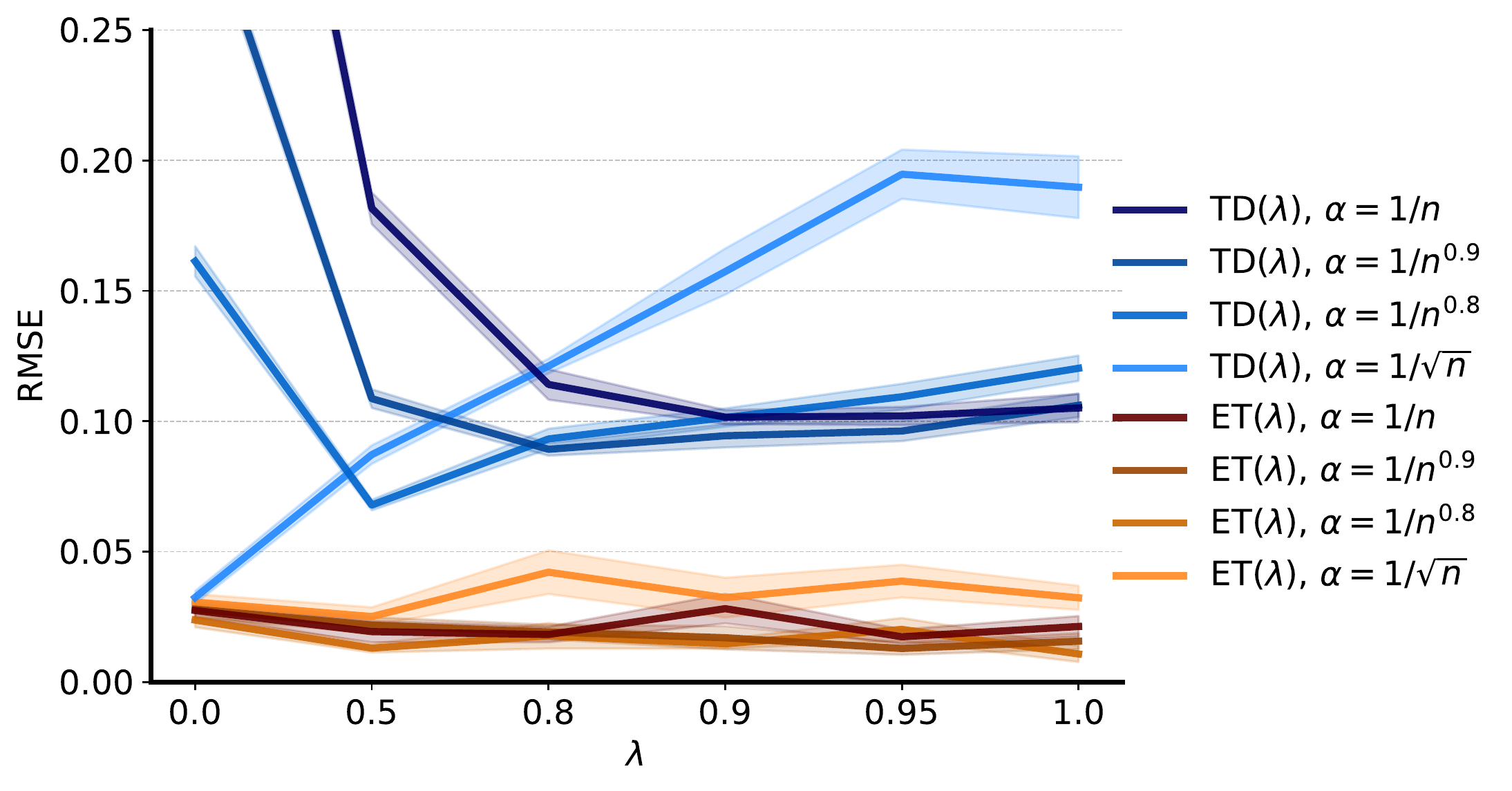}
    \caption{Comparison of prediction errors (lower is better) of TD($\lambda$) and ET($\lambda$) across different $\lambda$s and different step sizes in the multi-chain world \ref{multi_chain}. The data underpinning these plots is the same as the data used for Figure \ref{tabular}, with 32 parallel chains. In all cases the step size was $\alpha=1/n_t(S_t)^d$, where $n_t(s) = \sum_{i=0}^t I(S_i=s)$ is the number of visits to state $s$ in the first $t$ time steps, and where $d$ is a hyper-parameter.  Note that the step size is \emph{lower} when the exponent is \emph{higher}. We see that TD(0) performed best with a high step size, and that for high $\lambda$ lower step sizes performed better---TD(0) with the highest step size ($\alpha_t=1/\sqrt{n_t(S_t)}$) and TD(1) with the lowest step size ($\alpha_t=1/n_t(S_t)$) both performed poorly.  In contrast, ET($\lambda$) here performed well for any combination of step size and trace parameter $\lambda$.}
    \label{tabular_step_size}
\end{figure}

\subsection{Experiment details for Atari experiments}
\label{appendix:deep}
For our deep reinforcement learning experiments on Atari games, we compare to an implementation of online Q($\lambda$).  We first describe this algorithm, and then describe the expected-trace variant. All the Atari experiments were run with the ALE \citep{Bellemare:2013}, exactly as described in \citet{Mnih:2015}, including using action repeats (4x), downsampling (to $84 \times 84$), and frame stacking.  These experiments were conducted using Jax \citep{Jax:2018}.

In all cases, we used $\epsilon$-greedy exploration \citep[cf.][]{SuttonBarto:2018}, with an $\epsilon$ that quickly decayed from $1$ to $0.01$ according to $\epsilon_0 = 1$ and $\epsilon_t = \epsilon_{t-1} + 0.01(0.01 - \epsilon_{t-1})$. Unlike \citet{Mnih:2015}, we did not clip rewards, and we also did not apply any target normalisation \citep[cf.][]{vanHasselt:2016popart} or non-linear value transformations \citep{Pohlen:2018,vanHasselt:2019nonlinear}.  We conjecture that such extensions could be beneficial for performance, but they are orthogonal to the main research questions investigated here and are therefore left for future work.

\begin{algorithm}
    \caption{Q($\lambda$)}\label{alg:qlambda}
    \begin{algorithmic}[1]
        \State \text{initialise $\w$}
        \State \text{initialise $\e = \bm 0$}
        \State \text{observe initial state $S$}
        \State \text{pick action $A \sim \pi(q_{\w}(S))$}
        \State $v \gets \max_a q_{\w}(S, a)$ \label{q_lambda:td}
        \State $\g = 0$
        \Repeat{}
        \State \text{take action $A$, observe $R$, $\g'$ and $S'$} \hfill \# $\g' = 0$ on a terminating transition
        \State $v' \gets \max_a q_{\w}(S', a)$
        \State $\d \gets R + \g v' - v$
        \State $\e \gets \gamma \lambda \e + \nabla_{\w} q_{\w}(S, A)$ \label{qlambda:e}
        \State $\Delta\w \gets \delta \e + (v - q_{\w}(S, A)) \nabla_{\w} q_{\w}(S, A)$
        \State $\Delta\w \gets \text{transform}(\Delta \w)$ \hfill \# e.g., ADAM-ify \label{adamify}
        \State $\w \gets \w + \Delta\w$ \label{qlambda:update}
        \Until{done}
    \end{algorithmic}
\end{algorithm}

\subsubsection{Deep Q($\lambda$)}
We assume the typical setting \citep[e.g.,][]{Mnih:2015} where we have a neural network $\q_{\w}$ that outputs $|A|$ numbers, such that $q(s, a) = q_{\w}(s)[a]$.  That is, we forward the observation $s$ through network $q_{\w}$ with weights $\w$ and $|A|$ outputs, and then select the $a^{\text{th}}$ output to represent the value of taking action $a$.

Algorithm \ref{alg:qlambda} then works as follows.  For each transition, we first compute a telescoping TD error $\delta = r + \gamma' v' - v$ (line \ref{q_lambda:td}), where $\gamma' = 0$ on termination (and then $S'$ is the first observation of the next episode) and, in our experiments, $\gamma' = 0.995$ otherwise. We update the trace $\e$ as usual (line \ref{qlambda:e}), using accumulating traces.  Note that the weights and, hence, trace will also have elements corresponding to the weights of actions that were not selected.  The gradient with respect to those elements is considered to be zero, as is conventional.

Then, we compute a weight update $\Delta\w = \delta \e + (v - q_{\w}(S, A)) \nabla_{\w} q_{\w}(S, A)$.  The additional term corrects for the fact that our TD error is a telescoping error, and does not have the usual `${} - q(s, a)$' term.  This is akin to the Q($\lambda$) algorithm proposed by \citet{Peng:1996}.

Finally, we transform the resulting update, using a transformation exactly like ADAM \citep{Kingma:2015}, but applied to the update $\Delta\w$ rather than a gradient.  The hyper-parameters were $\beta_1 = 0.9$, $\beta_2 = 0.999$, and $\epsilon=0.0001$, and one of the step sizes as given in Figure \ref{deep_et}. We then apply the resulting transformed update by adding it to the weights (line \ref{qlambda:update}).

For the Atari experiments, we used the same preprocessing and network architecture as \citet{Mnih:2015}, except that we used 128 channels in each convolutional layer because we ran experiments on TPUs (version 3.0, using a single core per experiment) which are most efficient when using tensors where one dimension is a multiple of 128.  The experiments were written using JAX \citep{jax2018github} and Haiku \citep{haiku2020github}.

\subsubsection{Deep QET($\lambda$)}
We now describe the expected-trace algorithm, shown in Algorithm \ref{alg:qet}, which was used for the Atari experiments.  It is very similar to the Q($\lambda$) algorithm described above, and in fact equivalent when we set $\eta=1$.

The first main change is that we will split the computation of $q(s, a)$ into two separate parts, such that $q_{(\w, \bm{\xi})}(s, a) = \w_a\tr \x_{\bm{\xi}}(s)$.  This is equivalent to the previous algorithm: we have just labeled separate subsets of parameters as $(\w, \bm{\xi})$ rather than merging all of them into a single vector $\w$, and we have labeled the last hidden layer as $\x(s)$.  We keep separate traces for these subset (lines \ref{head_trace} and \ref{x_trace}), but this is equivalent to keeping one big trace for the combined set.
\begin{algorithm}
    \caption{QET($\lambda$)}\label{alg:qet}
    \begin{algorithmic}[1]
        \State \text{initialise $\w$, $\bm{\xi}$, $\th$}
        \State \text{initialise $\e = \bm 0$, $\y = \bm 0$}
        \State \text{observe initial state $S$}
        \State \text{pick action $A \sim \pi(q(S))$}
        \State $v \gets \max_a q_{(\w, \bm{\xi})}(S', a)$ \hfill \# $q_{(\w, \bm{\xi})}(s, a) = \w_a\tr \x_{\bm{\xi}}(s)$, where $\w = (\w_1, \ldots, \w_{|A|})$
        \State $\g = 0$
        \Repeat{}
        \State \text{take action $A$, observe $R$, $\g'$ and $S'$} \hfill \# $\g' = 0$ on any terminating transition
        \State $v' \gets \max_a q_{(\w, \bm{\xi})}(S', a)$
        \State $\d \gets R + \g v' - v$
        \State $\e^{\w} \gets \gamma \lambda \ee + \nabla_{\w} q_{(\w, \bm{\xi})}(S, A)$ \label{head_trace}
        \State $\e^{\bm{\xi}} \gets \gamma \lambda \e^{\bm{\xi}} + \nabla_{\bm{\xi}} q_{(\w, \bm{\xi})}(S, A)$ \label{x_trace}
        \State $\Delta\w \gets \delta \e^{\w} + (v - q_{(\w, \bm{\xi})}(S, A)) \nabla_{\w} q_{(\w, \bm{\xi})}(S, A)$
        \State $\Delta\bm{\xi} \gets \delta \e^{\bm{\xi}} + (v - q_{(\w, \bm{\xi})}(S, A)) \nabla_{\bm{\xi}} q_{(\w, \bm{\xi})}(S, A)$
        \State $\Delta\th \gets \nabla_{\th} \| \e^{\bm{\xi}} - \z_{\th}(S, A) \|^2_2$
        \State $\Delta\w \gets \text{transform}(\Delta\w)$ \hfill \# e.g., ADAM-ify
        \State $\Delta\bm{\xi} \gets \text{transform}(\Delta\bm{\xi})$
        \State $\Delta\th \gets \text{transform}(\Delta\th)$
        \State $\w \gets \w + \Delta\w$
        \State $\bm{\xi} \gets \bm{\xi} + \Delta\bm{\xi}$
        \State $\th \gets \th + \Delta\th$
        \State $\ee = (1 - \eta) \z_{\th}(s, a) + \eta \e^{\w}$
        \Until{done}
    \end{algorithmic}
\end{algorithm}

This split in parameters helps avoid learning an expected trace for the full trace, which has millions of elements.  Instead, we only learn expectations for traces corresponding to the last layer, denoted $\e^\w$.  Importantly, the function $\z_{\th}(s, a)$ should condition on both state and action.  This was implemented as a tensor $\th \in \reals^{|A| \times |A| \times |\x|}$, such that its tensor multiplication with the features $\x(s)$ yields a $|A| \times |A|$ matrix $\bm{Z}$.  Then, we interpret the vector $\z_a = [\bm{Z}]_a$ as the approximation to the expected trace $\E{ \e_t \mid S_t=s, A_t=a}$, and update it accordingly, using a squared loss (and, again, ADAM-ifying the update before applying it to the parameters).  The step size for the expected trace update was always $\beta=0.1$ in our experiments, and the expected trace loss was not back-propagated into the feature representation.  This can be done, but we leave any investigation of this for future work, as it would present a conflating factor for our experiments, because the expected trace update would then serve as an additional learning signal for the features that are also used for the value approximations.
\end{document}